\documentclass{article}

\usepackage{microtype}
\usepackage{graphicx}
\usepackage{graphics}
\usepackage{mathtools}
\usepackage{booktabs} % for professional tables
\usepackage{amsthm}
\usepackage{amssymb}
\usepackage{bm}
\usepackage{amsfonts}
\usepackage{amsmath}
\usepackage{bbm}
\usepackage{subcaption}
\usepackage{adjustbox}

\def\eqref#1{equation~\ref{#1}}
\def\1{\bm{1}}

\def\eps{{\epsilon}}

\def\mA{{\bm{A}}}

\def\mS{{\bm{S}}}

\DeclareMathAlphabet{\mathsfit}{\encodingdefault}{\sfdefault}{m}{sl}
\SetMathAlphabet{\mathsfit}{bold}{\encodingdefault}{\sfdefault}{bx}{n}

\def\gA{{\mathcal{A}}}
\def\gB{{\mathcal{B}}}

\def\gN{{\mathcal{N}}}

\def\gU{{\mathcal{U}}}

\def\sR{{\mathbb{R}}}

\def\sX{{\mathbb{X}}}

\newcommand{\R}{\mathbb{R}}

\usepackage{hyperref}

\theoremstyle{definition}
\newtheorem{definition}{Definition}[section]
\newtheorem{proposition}{Proposition}[section]
\newtheorem{lemma}{Lemma}[section]
\newtheorem{corollary}{Corollary}[lemma]

\usepackage[accepted]{icml2020}

\icmltitlerunning{Deep Graph Mapper}

\begin{document}

\twocolumn[
\icmltitle{Deep Graph Mapper: Seeing Graphs through the Neural Lens}


\icmlsetsymbol{equal}{*}

\begin{icmlauthorlist}
\icmlauthor{Cristian Bodnar}{equal,cam}
\icmlauthor{C\u{a}t\u{a}lina Cangea}{equal,cam}
\icmlauthor{Pietro Li\`{o}}{cam}
\end{icmlauthorlist}

\icmlaffiliation{cam}{Department of Computer Science and Technology, University of Cambridge, Cambridge, United Kingdom}

\icmlcorrespondingauthor{Cristian Bodnar}{cb2015@cam.ac.uk}
\icmlcorrespondingauthor{C\u{a}t\u{a}lina Cangea}{Catalina.Cangea@cst.cam.ac.uk}



\icmlkeywords{Machine Learning, Deep Learning, Topology, Graph Neural Networks, Graph Pooling, Graph Classification, Visualisation, Clustering}

\vskip 0.3in
]



\printAffiliationsAndNotice{\icmlEqualContribution} % otherwise use the standard text.

\begin{abstract}
Recent advancements in graph representation learning have led to the emergence of condensed encodings that capture the main properties of a graph. However, even though these abstract representations are powerful for downstream tasks, they are not equally suitable for visualisation purposes. In this work, we merge Mapper, an algorithm from the field of Topological Data Analysis (TDA), with the expressive power of Graph Neural Networks (GNNs) to produce hierarchical, topologically-grounded visualisations of graphs. These visualisations do not only help discern the structure of complex graphs but also provide a means of understanding the models applied to them for solving various tasks. We further demonstrate the suitability of Mapper as a topological framework for graph pooling by mathematically proving an equivalence with Min-Cut and Diff Pool. Building upon this framework, we introduce a novel pooling algorithm based on PageRank, which obtains competitive results with state of the art methods on graph classification benchmarks.
\end{abstract} 

\section{Introduction}

Tasks involving graph-structured data have received much attention lately, due to the abundance of relational information in the real world. Considerable progress has been made in the field of graph representation learning through generating graph encodings with the help of deep learning techniques. The abstract representations obtained by these models are typically intended for further processing within downstream tasks. However, few of these advancements have been directed towards visualising and aiding the human understanding of the complex networks ingested by machine learning models. We believe that data and model visualisation are important steps of the statistical modelling process and deserve an increased level of attention. 

\begin{figure}
    \centering
    \includegraphics[width=1.0\columnwidth]{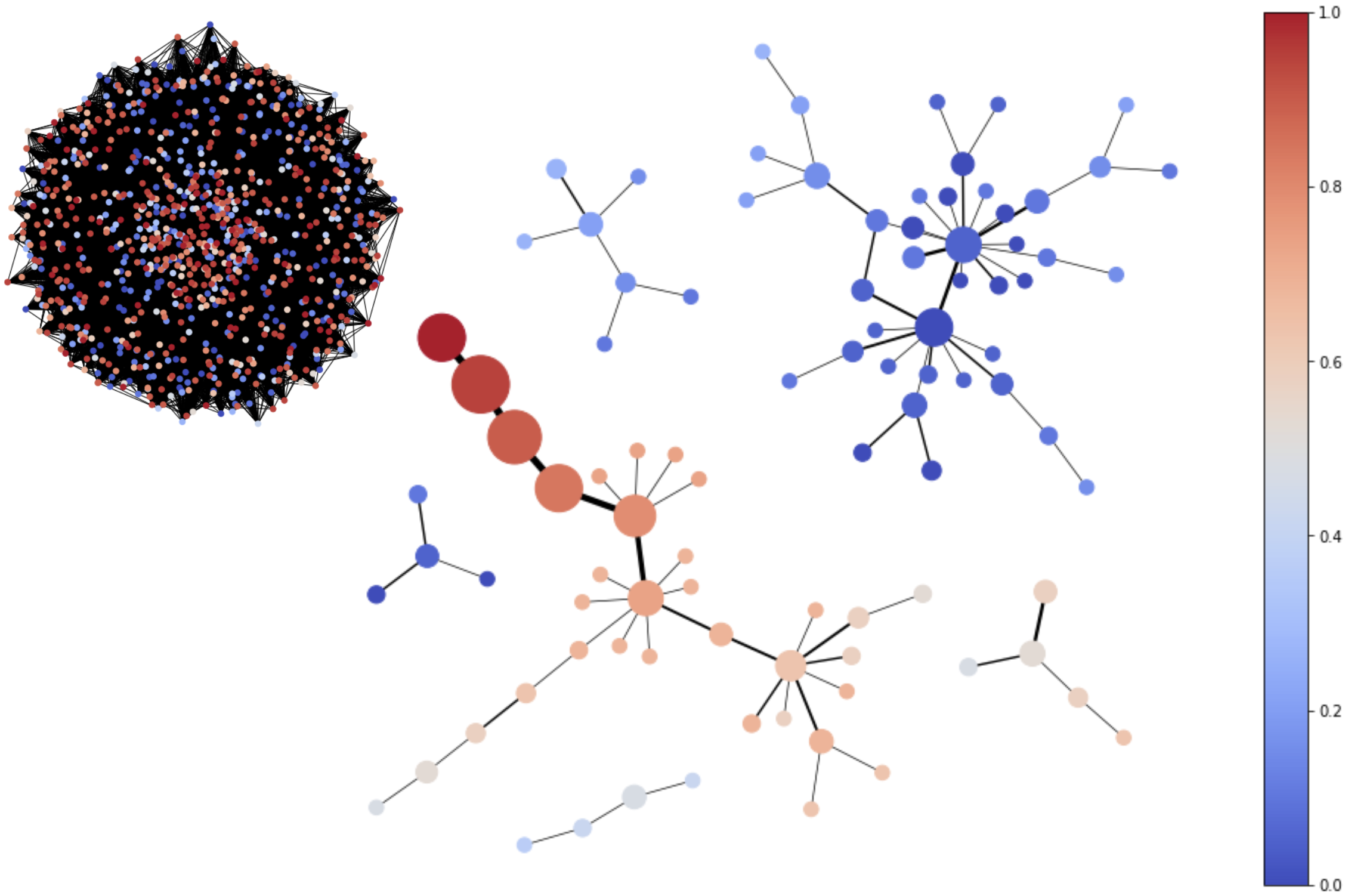}
    \caption{A Deep Graph Mapper (DGM) visualisation of a dense graph containing spammers and non-spammers (top-left). DGM removes the visual clutter in the original NetworkX plot (using a Graphviz `spring' layout), by providing an informative summary of the graph. Each node in the DGM graph represents a cluster of nodes in the original graph. The size of the DGM nodes is proportional to the number of nodes in the corresponding cluster. Each edge signifies that two clusters have overlapping nodes proportional to the thickness of the edge. The clusters are mainly determined by a neural `lens' function: a GCN performing binary node classification. The  colorbar indicates the GCN predicted probability that a node is a spammer. The DGM visualisation illustrates important features of the graph: \emph{spammers} (red) are highly inter-connected and consequently grouped in a just few large clusters, whereas \emph{non-spammers} (blue) are less connected to the rest of the graph and thus form many small clusters.}
    \label{fig:intro}
\end{figure}

Here, we tackle this problem by merging Mapper~\citep{Singh2007}, an algorithm from the field of Topological Data Analysis (TDA)~\citep{chazal2017introduction}, with the demonstrated representational power of Graph Neural Networks (GNNs)~\citep{scarselli2008graph, battaglia2018relational, bronstein2017geometric} and refer to this synthesis as Deep Graph Mapper (DGM). Our method offers a means to visualise graphs and the complex data living on them through a GNN `lens'. Moreover, the aspects highlighted by the visualisation can be flexibly adjusted via the loss function of the network. Finally, DGM achieves progress with respect to GNN interpretability, providing a way to visualise the model and identify the mistakes it makes for node-level supervised and unsupervised learning tasks.

\begin{figure*}[ht]
    \centering
    \includegraphics[width=0.9\textwidth]{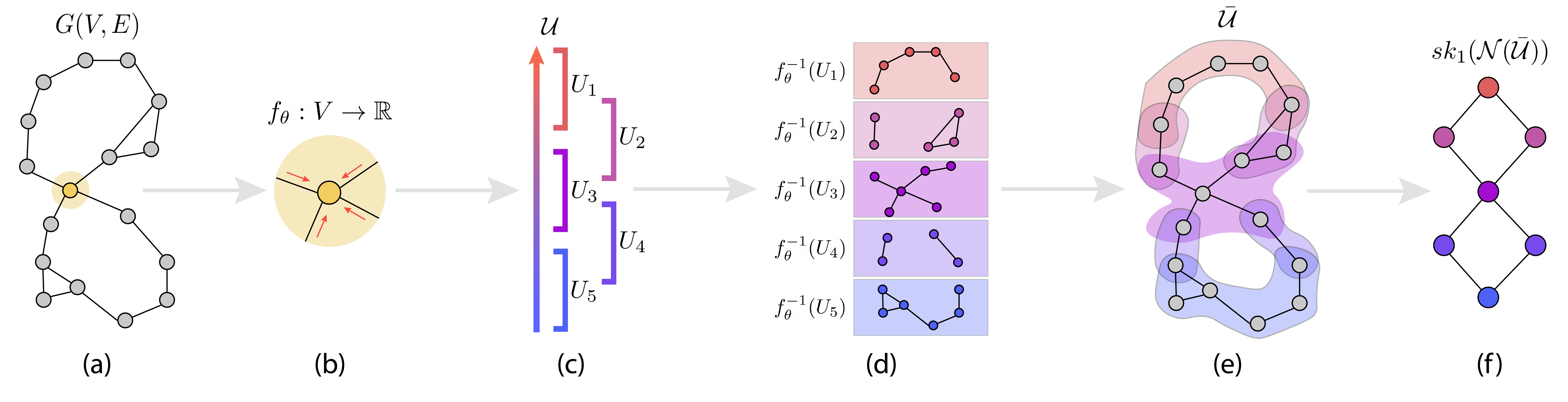}
    \caption{A cartoon illustration of The Deep Graph Mapper (DGM) algorithm where, for simplicity, the GNN approximates a `height' function over the nodes in the plane of the diagram. The input graph (a) is passed through a Graph Neural Network (GNN), which maps the vertices of the graph to a real number (the height) (b). Given a cover $\gU$ of the image of the GNN (c), the refined pull back cover $\Bar{\gU}$ is computed (d--e). The 1-skeleton of the nerve of the pull back cover provides the visual summary of the graph (f). The diagram is inspired from \citet{hajij2018mapper}.}
    \label{fig:mapper}
\end{figure*}

We then demonstrate that Mapper graph summaries are not only suitable for visualisation, but can also constitute a pooling mechanism within GNNs. We begin by proving that Mapper is a generalisation of pooling methods based on soft cluster assignments, which include state-of-the-art algorithms like minCUT~\citep{bianchi2019mincut} and DiffPool~\citep{ying2018hierarchical}. Building upon this topological perspective, we propose MPR, a novel graph pooling algorithm based on PageRank~\citep{page1999pagerank}. Our method obtains competitive or superior results when compared with state-of-the-art pooling methods on graph classification benchmarks. To summarise, our contributions are threefold: 
\begin{itemize}
    \item DGM, a topologically-grounded method for visualising graphs and the GNNs applied to them. 
    \item A proof that Mapper is a generalisation of soft cluster assignment pooling methods, including the state-of-the-art minCUT and Diff pool.
    \item MPR, a Mapper-based pooling method that achieves similar or superior results compared to state-of-the-art methods on several graph classification benchmarks.
\end{itemize}

\section{Related Work}

\subsection{Graph Visualisation and Interpretability} \label{graphvizinterp}

A number of software tools exist for visualising node-link diagrams: NetworkX~\citep{hagberg2008exploring}, Gephi~\citep{ICWSM09154}, Graphviz~\citep{gansner2000open} and NodeXL~\citep{smith2010nodexl}. However, these tools do not attempt to produce condensed summaries and consequently suffer on large graphs from the visual clutter problem illustrated in Figure \ref{fig:intro}. This makes the interpretation and understanding of the graph difficult. Some of the earlier attempts to produce visual summaries rely on grouping nodes into a set of predefined motifs \citep{motif_simplification} or compressing them in a lossless manner into modules \citep{module_simplification}. However, these mechanisms are severely constrained by the simple types of node groupings that they allow.

Mapper-based summaries for graphs have recently been considered by~\citet{hajij2018mapper}. However, despite the advantages provided by Mapper, their approach relies on hand-crafted graph-theoretic `lenses', such as the average geodesic distance, graph density functions or eigenvectors of the graph Laplacian. Not only are these methods rigid and unable to adapt well to the graph or task of interest, but they are also computationally inefficient. Moreover, they do not take into account the features of the graph. In this paper, we build upon their work by considering learnable functions (GNNs) that do not present these problems. 

Mapper visualisations are also an indirect way to analyse the behaviour of the associated `lens' function. However, visualisations that are directly oriented towards model interpretability have been recently considered by~\citet{ying2019gnnexplainer}, who propose a model capable of indicating the relevant sub-graphs and features for a given model prediction. 

\subsection{Graph Pooling}
\label{sec:rel_pooling}

Pooling algorithms have already been considerably explored within GNN frameworks for graph classification. \citet{luzhnica2019clique} propose a topological approach to pooling which coarsens the graph by aggregating its maximal cliques into new clusters. However, cliques are local topological features, whereas our MPR algorithm leverages a global perspective of the graph during pooling. Two paradigms distinguish themselves among learnable pooling layers: top-$k$ pooling based on a learnable ranking, initially adopted by~\citet{gao2019graph} (Graph U-Nets), and learning the cluster assignment~\citep{ying2018hierarchical} with additional entropy and link prediction losses for more stable training (DiffPool). Following these two trends, several variants and incremental improvements have been proposed. The top-$k$ approach is explored in conjunction with jumping-knowledge networks~\citep{cangea2018towards}, attention~\citep{lee2019self, huang2019attpool} and self-attention for cluster assignment~\citep{ranjan2019asap}. Similarly to DiffPool, the method suggested by \citet{bianchi2019mincut} uses several loss terms to enforce clusters with strongly connected nodes, similar sizes and orthogonal assignments. A different approach is also proposed by~\citet{ma2019graph}, who leverage spectral clustering for pooling.

\subsection{Topological Data Analysis in Machine Learning}
\label{sec:tdaml}

Persistent homology \citep{ph_survey} has been so far the most popular branch of TDA applied to machine learning and graphs, especially. \citet{top_signatures} and \citet{carriere2019perslay} integrated graph persistence diagrams with neural networks to obtain topology-aware models. A more advanced approach for GNNs has been proposed by  \citet{graph_filtration_learning}, who backpropagate through the persistent homology computation to directly learn a graph filtration.

Mapper, another central algorithm in TDA, has been used in deep learning almost exclusively as a tool for understanding neural networks. \citet{topology_of_learning} use Mapper to visualise the evolution of the weights of fully connected networks during training, while \citet{mapper_cnn} use it to visualise the filters computed by CNNs. To the best of our knowledge, the paper of \citet{hajij2018mapper} remains the only application of Mapper on graphs.

\section{Mapper for Visualisations}

In this section we describe the proposed integration between Mapper and GCNs. 

\subsection{Mapper on Graphs} \label{mappergraphs}

We start by reviewing Mapper \citep{Singh2007}, a topologically-motivated algorithm for high-dimensional data visualisation. Intuitively, Mapper obtains a low-dimensional image of the data that can be easily visualised. The algorithm produces an output graph that shows how clusters within the data are semantically related from the perspective of a `lens' function. The resulting graph preserves the notion of `nearness' of the input topological space, but it can compress large scale distances by connecting far-away points that are similar according to the lens. We first introduce the required mathematical background.

\begin{definition}
An \textbf{open cover} of a topological space $X$ is a collection of open sets $(U_i)_{i \in I}$, for some indexing set $I$, whose union includes $X$.
\end{definition}

For example, an open cover for the real numbers could be $\{(-\infty, 0), (-2, 3), (1, \infty) \}$. Similarly, $\{ \{v_1, v_2, v_3\}, \{v_4\}\}$ is an open cover for a set of vertices $\{v_1, v_2, v_3, v_4\}$. 

\begin{definition}
Let $X$ be a topological space, $f: X \to \sR^d, d \geq 1$ a continuous function, and $\gU = (U_i)_{i \in I}$ a cover of $\R^d$. Then, the \textbf{pull back cover} $f^*(\gU)$ of $X$ induced by $(f, \gU)$ is the collection of open sets $f^{-1}(U_i), i \in I$, for some indexing set $I$, where by $f^{-1}(U_i)$ we denote the preimage of the set $U_i$. 
\end{definition}

Given a dataset $\sX$, a carefully chosen lens function $f(\sX)$ and cover $\gU$, Mapper first computes the associated pull back cover $f^*(\gU)$. Then, using a clustering algorithm of choice, it clusters each of the open sets $f^{-1}(U_i)$ in $f^*(\gU)$. The resulting group of sets is called the \textbf{refined pull back cover}, denoted by $\Bar{\gU} = (\bar{U_j})_{{j \in J}}$ with indexing set $J$. Concretely, in this paper, the input dataset is a weighted graph $\sX = G(V, E)$ and $f:V \to \R^d$ is a function over the vertices of the graph. An illustration for these steps is provided in Figure \ref{fig:mapper} (a-e) for a `height' function $f$.

Finally, Mapper produces a new graph by taking the \textbf{$1$-skeleton of the nerve $\gN(\Bar{\gU})$ of the refined pull back cover}: a graph where the vertices are given by $(v_j)_{j \in J}$ and two vertices $v_{j_1}, v_{j_2}$ are connected if and only if $\bar{U}_{j_1} \cap \bar{U}_{j_2} \neq \emptyset$. Informally, the soft clusters formed by the refined pull back become the new nodes and clusters with common nodes are connected by an edge. This final step is illustrated in Figure \ref{fig:mapper} (f).  

Three main degrees of freedom that determine the visualisation can be distinguished within Mapper:

\textbf{The lens $f$}: In our case, the lens $f: V \to \sR^d$ is a function over the vertices, which acts as a filter that emphasises certain features of the graph. The choice of $f$ highly depends on the properties to be highlighted by the visualisation. We also refer to the co-domain of $f$ as the parametrisation space. 

\textbf{The cover $\gU$}: The choice of cover determines the resolution of the summary. Fine-grained covers will produce more detailed visualisations, while higher overlaps between the sets in the cover will increase the connectivity of the output graph. When $d = 1$, a common choice is to select a set of overlapping intervals over the real line, as in Figure \ref{fig:mapper} (c).  

\textbf{Clustering algorithm}: Mapper works with any clustering algorithm. When the dataset is a graph, a natural choice adopted by \citet{hajij2018mapper} is to take the connected components of the subgraphs induced by the vertices $f^{-1}(U_i), i \in I$ (Figure \ref{fig:mapper} (e-f)). This is also the approach we follow, unless otherwise stated.

\subsection{Seeing through the Lens of GCNs}

As mentioned in Section~\ref{graphvizinterp}, \citet{hajij2018mog} have considered a set of graph theoretic functions for the lens. However, with the exception of PageRank, all of these functions are difficult to compute on large graphs. The average geodesic distance and the graph density function require computing the distance matrix of the graph, while the eigenfunction computations do not scale beyond graphs with a few thousands of nodes. Besides their computational complexity, many real-world graphs contain features within the nodes and edges of the graphs, which are ignored by graph-theoretic summaries. 

In this work, we leverage the recent progress in the field of graph representation learning and propose a series of lens functions based on Graph Convolutional Networks (GCNs)~\citep{kipf2016semi}. We refer to this integration between Mapper and GCNs as Deep Graph Mapper (DGM). 

Unlike graph-theoretic functions, GCNs can naturally learn and integrate the features associated with the graph and its topological properties, while also scaling to large, complex graphs. Additionally, visualisations can flexibly be tuned for the task of interest by adjusting the associated loss function. 

Mapper further constitutes a method for implicitly visualising the lens function, so DGM is also a novel approach to model understanding. Figure~\ref{fig:ground-truth} illustrates how Mapper can be used to identify mistakes that the underlying GCN model makes in node classification. This showcases the potential of DGM for continuous model refinement. 

\subsection{Supervised Lens}

A natural application for DGM in the supervised domain is as an assistive tool for binary node classification. The node classifier, a function $f: V \to [0, 1]$, is an immediate candidate for the DGM lens. The usual choice for a cover $\gU$ over the real numbers is a set of $n$ equally-sized overlapping intervals, with an overlap percentage of $g$. This often produces a hierarchical (tree-like) perspective over the input graph, as shown in Figure~\ref{fig:intro}.

\begin{figure}[ht]
    \centering
    \includegraphics[width=0.7\columnwidth]{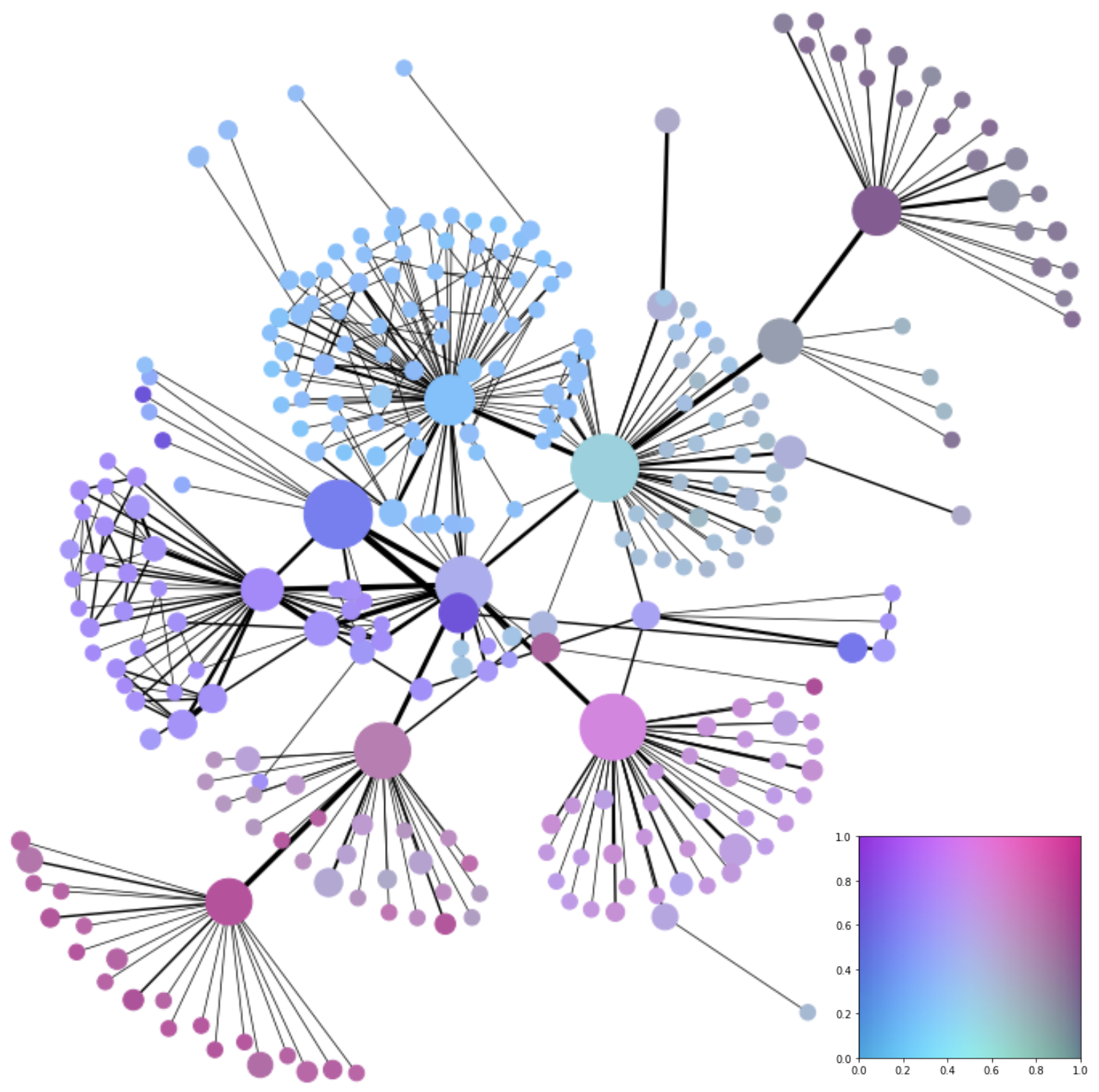}
    \caption{DGM visualisation of the Cora dataset, using a GCN classifier as the lens and a grid-cover $\gU$ with 9 cells and 10\% overlap across each axis. The two-dimensional cover is able to capture more complex relations of semantic proximity, such as high-order cliques. We use a two-dimensional colormap (bottom-right) to reflect the relationship between the nodes and their position in the parametrisation space.}
    \label{fig:cora_2d_supervised}
\end{figure}

However, most node-classification tasks involve more than two labels. A simple solution in this case is to use a dimensionality reduction algorithm such as $t$-SNE~\citep{vanDerMaaten2008} to embed the logits in $\sR^d$, where $d$ is small. Empirically, when the number of classes is larger than two, we find a 2D parametrisation space to better capture the relationships between the classes in the semantic space. Figure \ref{fig:cora_2d_supervised} includes a visualisation of the Cora dataset using a $t$-SNE embedding of the logits. Faster dimensionality reduction methods such as PCA~\citep{Jolliffe} or the recently proposed NCVis~\citep{aleks2020ncvis} could be used to scale this approach to large graphs.

\begin{figure}
    \centering
    \begin{subfigure}[t]{0.48\columnwidth}
         \centering
         \includegraphics[width=\columnwidth]{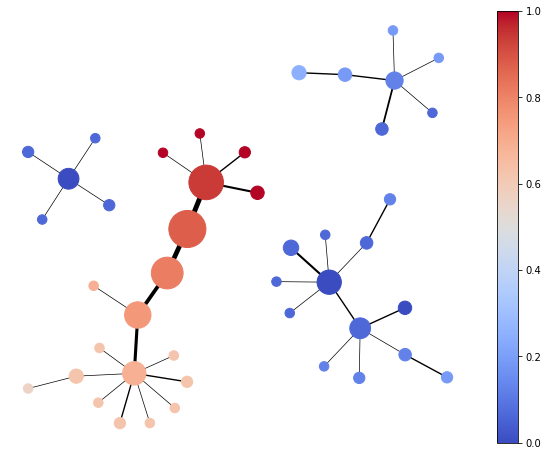}
         \caption{DGM visualisation.}
    \end{subfigure}
    ~
    \begin{subfigure}[t]{0.48\columnwidth}
         \centering
         \includegraphics[width=\columnwidth]{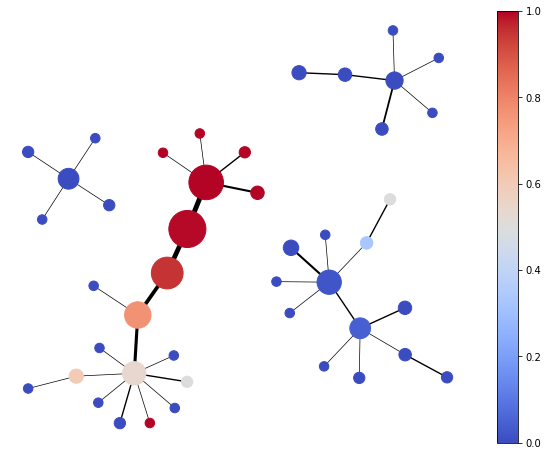}
         \caption{DGM visualisation with ground-truth colouring.}
    \end{subfigure}
    \caption{Side-by-side comparison of the DGM visualisation and the one with ground-truth labelling. The two images clearly highlight certain mistakes that the classifier is making, such as the many small nodes in the bottom-left corner that the classifier tends to label as spam (light-red), even though most of them are non-spam users (in blue).}
    \label{fig:ground-truth}
\end{figure}

In a supervised setting, DGM visualisations can also be integrated with the ground-truth labels of the nodes. The latter provide a means for visualising both the mistakes that the classifier is making and the relationships between classes. For lens trained to perform binary classification of the nodes, we colour the nodes of the summary graph proportionally to the number of nodes in the original graph that are part of that cluster and belong to the positive class. For lens that are multi-label classifiers, we colour each node with the most frequent class in its corresponding cluster. Figure~\ref{fig:ground-truth} gives the labelled summary for the binary Spam dataset, while Figure~\ref{fig:qualitative} includes two labelled visualisations for the Cora and CiteSeer datasets. 

\subsection{Unsupervised Lens}

The expressive power of GNNs is not limited to supervised tasks. In fact, many graphs do not have any labels associated with their nodes---in this case, the lens described so far, which require supervised training, could not be applied. However, the recent models for unsupervised graph representation learning constitute a powerful alternative.

Here, we use Deep Graph Infomax (DGI)~\citep{velivckovic2018deep} to compute node embeddings in $\sR^d$ and obtain a low-dimensional parametrisation of the graph nodes. DGI computes node-level embeddings by learning to maximise the mutual information (MI) between patch representations and corresponding high-level summaries of graphs. We have empirically found that applying $t$-SNE over a higher-dimensional embedding of DGI works better than learning a low-dimensional parametrisation with DGI directly. Figure~\ref{fig:qualitative} includes two labelled visualisations obtained with the unsupervised DGI lens on Cora and CiteSeer.

\begin{figure}[ht]
    \centering
    \begin{subfigure}[t]{0.31\columnwidth}
         \centering
         \includegraphics[width=\textwidth]{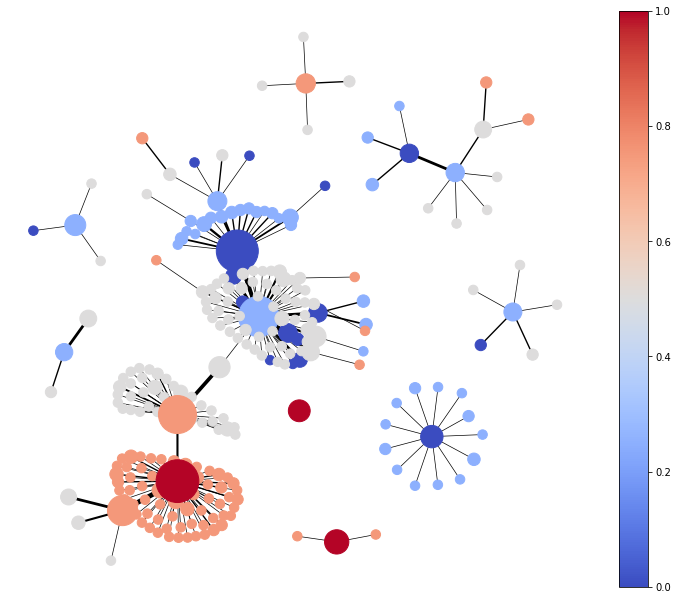}
         \caption{$n=5, g=0.2$}
    \end{subfigure}
    ~
    \begin{subfigure}[t]{0.31\columnwidth}
         \centering
         \includegraphics[width=\textwidth]{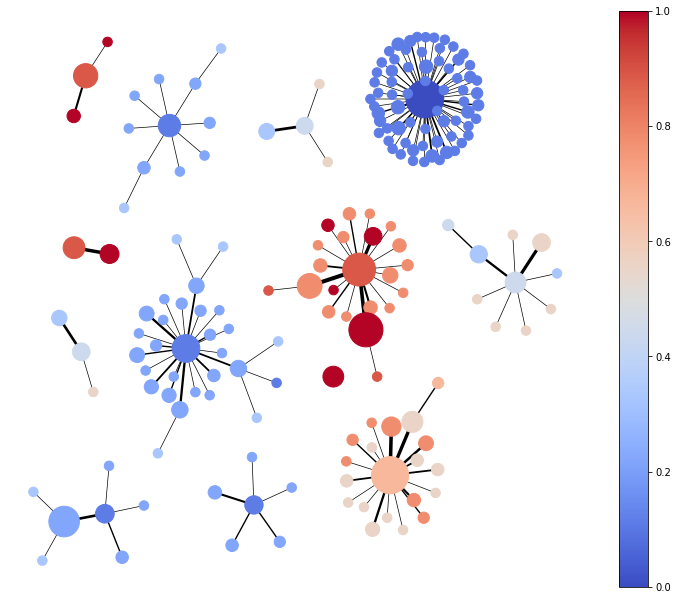}
         \caption{$n=10, g=0.2$}
    \end{subfigure}
    ~
    \begin{subfigure}[t]{0.31\columnwidth}
         \centering
         \includegraphics[width=\textwidth]{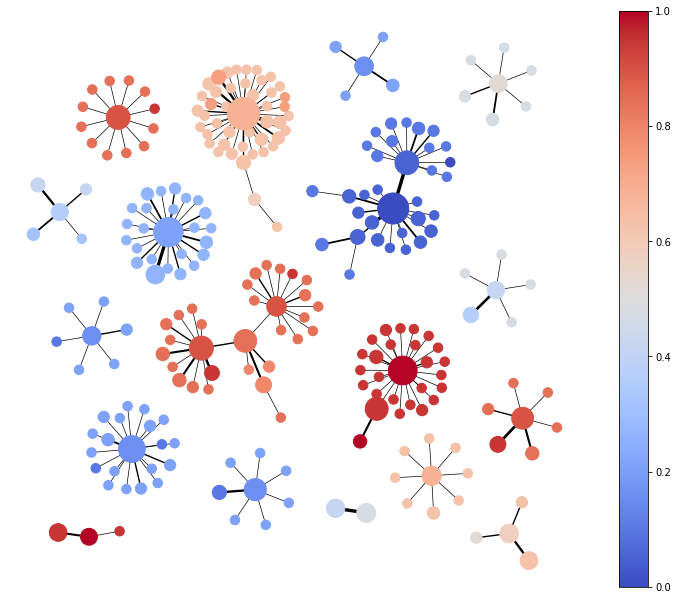}
         \caption{$n=20, g=0.2$}
    \end{subfigure}
    ~
    \begin{subfigure}[t]{0.31\columnwidth}
         \centering
         \includegraphics[width=\textwidth]{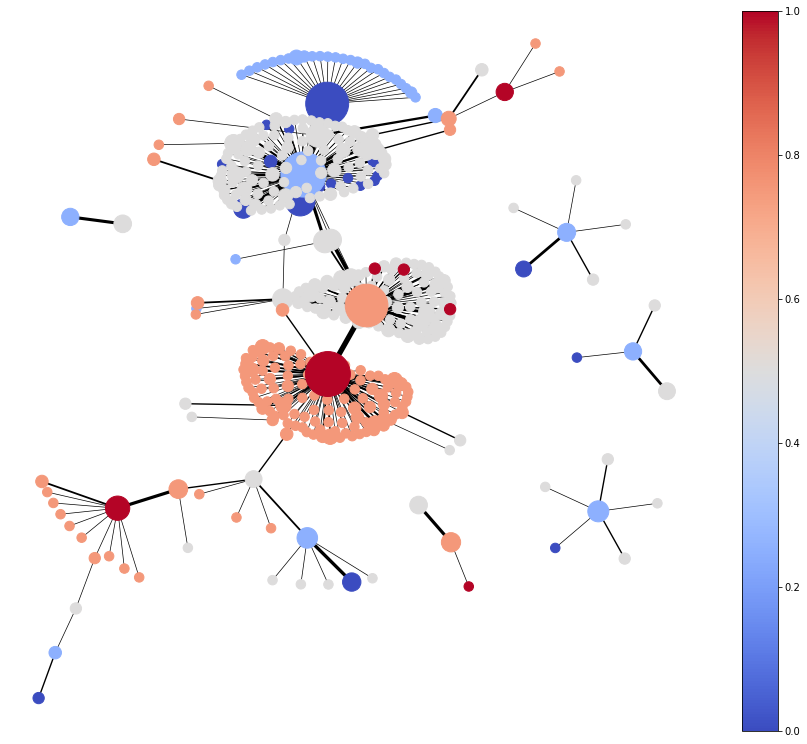}
         \caption{$n=5, g=0.4$}
    \end{subfigure}
    ~
    \begin{subfigure}[t]{0.31\columnwidth}
         \centering
         \includegraphics[width=\textwidth]{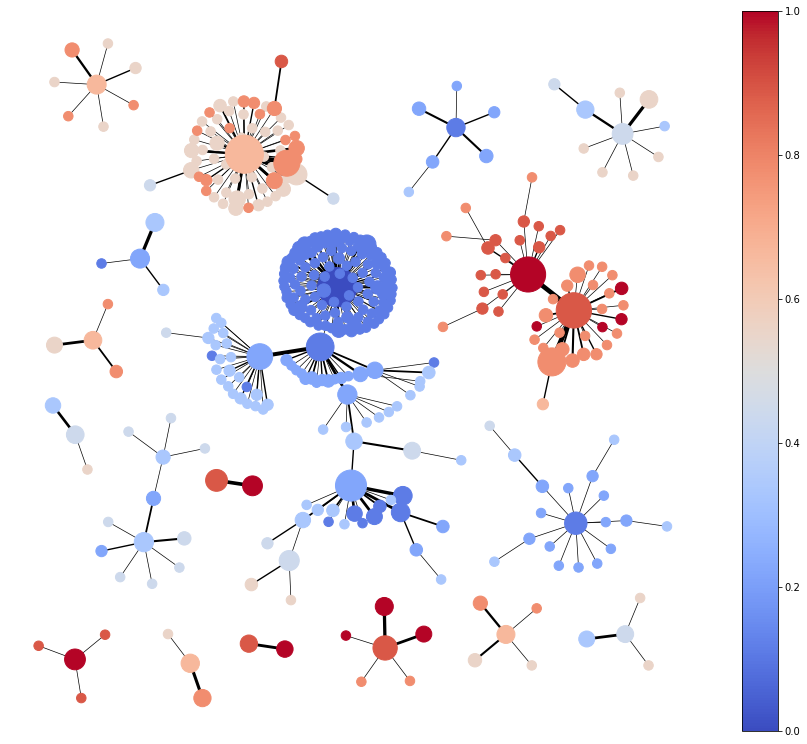}
         \caption{$n=10, g=0.4$}
    \end{subfigure}
    ~
    \begin{subfigure}[t]{0.31\columnwidth}
         \centering
         \includegraphics[width=\textwidth]{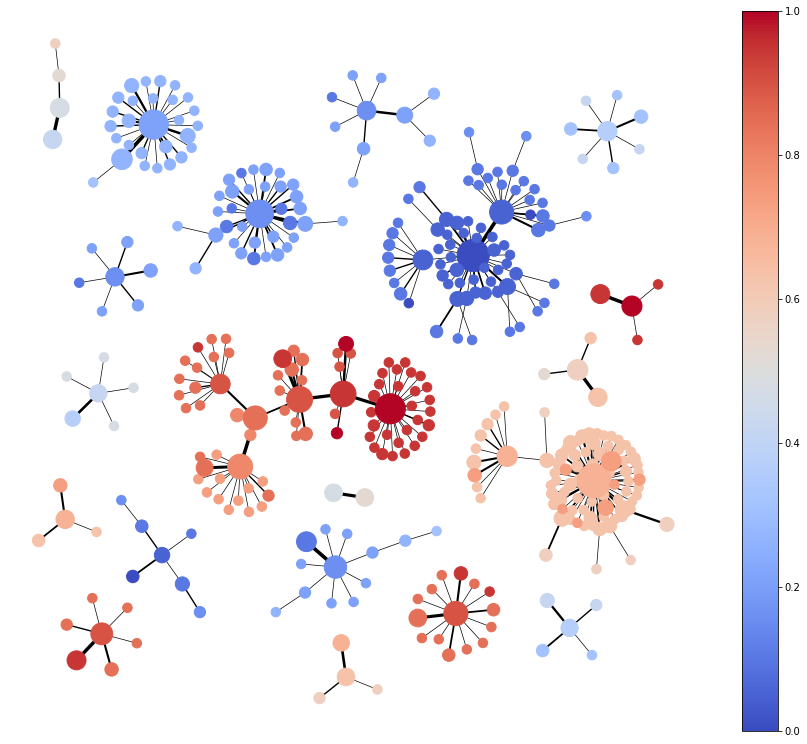}
         \caption{$n=20, g=0.4$}
    \end{subfigure}
    \caption{By adjusting the cover $\gU$, multi-resolution visualisations of the graph can be obtained. The multi-resolution perspective can help one identify the persistent features of the graph (i.e. the features that survive at multiple scales). For instance, the connection between the grey and blue nodes does not survive at all resolutions, whereas the connection between the red and orange nodes persists.}
    \label{fig:hierarhical_citeseer}
\end{figure}

\subsection{Hierarchical Visualisations}

One of the biggest advantages of Mapper is the flexibility to adjust the resolution of the visualisation by modifying the cover $\gU$. The number of sets in $\gU$ determines the coarseness of the output graph, with a larger number of sets generally producing a larger number of nodes. The overlap between these sets determines the level of connectivity between the nodes---for example, a higher degree of overlap produces denser graphs. By adjusting these two parameters, one can discover what sort of properties of the graph are persistent, consequently surviving at multiple scales, and which features are mainly due to noise. In Figure \ref{fig:hierarhical_citeseer}, we include multiple visualisations of the CiteSeer~\citep{sen2008collective} dataset at various scales---these are determined by adjusting the cover $\gU$ via its size $n$ and interval overlap percentage $g$.

\subsection{The Dimensionality of the Output Space}

\begin{figure}[ht]
    \centering
    \includegraphics[width=0.8\columnwidth]{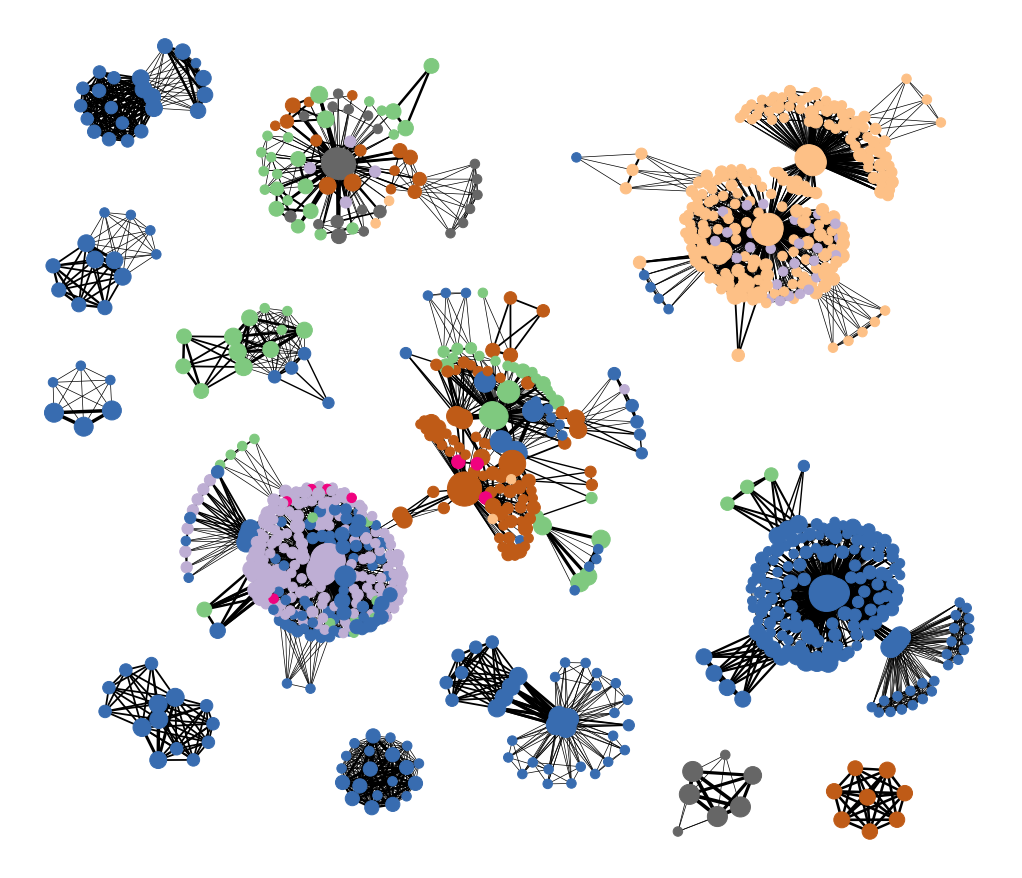}
    \caption{DGM visualisation of Cora with a 3D parametrisation and ground-truth labels. Each colour represents a different class. While higher-dimensional parametrisations can encode more complex semantic relations, the interpretability of the lens becomes more difficult.}
    \label{fig:cora_3d_labeled.}
\end{figure}

Although Mapper typically uses a one-dimensional parametric space, higher-dimensional ones have extra degrees of freedom in embedding the nodes inside them. Therefore, higher-dimensional spaces feature more complex neighbourhoods or relations of semantic proximity, as it can be seen from Figures \ref{fig:intro} (1D), \ref{fig:cora_2d_supervised} (2D) and \ref{fig:cora_3d_labeled.} (3D). 

However, the interpretability of the lens decreases as $d$ increases, for a number of reasons. First, colormaps can be easily visualised in 1D and 2D, but it is hard to visualise a 3D-colormap. Therefore, for the 3D parametrisation in Figure \ref{fig:cora_3d_labeled.}, we use dataset labels to colour the nodes. Second, the curse of dimensionality becomes a problem and the open sets of the cover are likely to contain fewer and fewer nodes as $d$ increases. Therefore, the resolution of the visualisation can hardly be adjusted in higher dimensional spaces. 

\section{Mapper for Pooling}

We now present a theoretical result showing that the graph summaries computed by Mapper are not only useful for visualisation purposes, but also as a pooling (graph-coarsening) mechanism inside GNNs. Building on this evidence, we introduce a pooling method based on PageRank and Mapper. 

\subsection{Mapper and Spectral Clustering}

The relationship between Mapper for graphs and spectral clustering has been observed by \citet{hajij2018mapper}. This link is a strong indicator that Mapper can compute `useful' clusters for pooling. We formally restate this observation below and provide a short proof.

\begin{proposition}
Let $L$ be the Laplacian of a graph $G(V, E)$ and $l_2$ the eigenvector corresponding to the second lowest eigenvalue of $L$, also known as the Fiedler vector~\citep{fiedler1973algebraic}. Then, for a function $f: V \to \sR, f(v) = l_2(v)$, outputting the entry in the eigenvector $l_2$ corresponding to node $v$ and a cover $ \gU = \{(-\infty, \eps), (-\eps, +\infty)\}$, Mapper produces a spectral bi-partition of the graph for a sufficiently small positive $\eps$.
\end{proposition}

\begin{proof}
It is well known that the Fiedler vector can be used to obtain a ``good'' bi-partition of the graph based on the signature of the entries of the vector (i.e. $l_2(v) > 0$ and $l_2(v) < 0$) (please refer to~\citet{berkeley_notes} for a proof). Therefore, by setting $\eps$ to a sufficiently small positive number $\eps < \min_v |l_2(v)|$, the obtained pull back cover is a spectral bi-partition of the graph. 
\end{proof}

The result above indicates that Mapper is a generalisation of spectral clustering. As the latter is strongly related to min-cuts~\citep{stanford_notes}, the proposition also provides a link between Mapper and min-cuts.

\subsection{Mapper and Soft Cluster Assignments}

Let $G(V, E)$ be a graph with self-loops for each node and adjacency matrix $\mathbf{A}$. \emph{Soft cluster assignment pooling methods} use a soft cluster assignment matrix $\mathbf{S} \in \sR^{N \times K}$, where $N$ is the number of nodes in the graph and $K$ is the number of clusters, and compute the new adjacency matrix of the graph via $\mathbf{A}' = \mathbf{S}^T\mathbf{A}\mathbf{S}$. Equivalently, two soft clusters become connected if and only if there is a common edge between them.

\begin{proposition}
There exists a graph $G'(V', E')$ derived from $G$, a soft cluster assignment $\mathbf{S}'$ of $G'$ based on $\mathbf{S}$, and a cover $\gU$ of $\bigtriangleup_{K-1}$, such that the 1-skeleton of the nerve of the pull back cover induced by $(\mathbf{S}', \gU)$ is isomorphic with the graph defined by $\mathbf{A}'$.
\end{proposition}

This shows that Mapper is a generalisation of soft-cluster assignment methods. A detailed proof and a diagrammatic version of it can be found in the supplementary material. Note that this result implicitly uses an instance of Mapper with the simplest possible clustering algorithm that assigns all vertices in each open set of the pull back cover to the same cluster (same as no clustering).

We hope that this result will enable theoreticians to study pooling operators through the topological and statistical properties of Mapper~\citep{carriere2018statistical}. At the same time, we encourage practitioners to take advantage of it and design new pooling methods in terms of a well-chosen lens function $f$ and cover $\gU$ for its image. 

A remaining challenge is designing differentiable pull back cover operations to automatically learn a parametric lens and cover. We leave this exciting direction for future work and focus on exploiting the proven connection---we propose in the next section a simple, non-differentiable pooling method based on PageRank and Mapper that performs surprisingly well.

Additionally, we use this result to propose \textbf{Structural DGM} (SDGM), a version of DGM where edges represent structural connections between the clusters in the refined pull back cover, rather than semantic connections. SDGM can be found in Appendix B. 

\subsection{PageRank Pooling}

\subsubsection{Model}

For the graph classification task, each example $\text{G}$ is represented by a tuple $(\mathbf{X}, \mathbf{A})$, where $\mathbf{X}$ is the node feature matrix and $\mathbf{A}$ is the adjacency matrix. Both our graph embedding and classification networks consist of a sequence of graph convolutional layers~\cite{kipf2016semi}; the $l$-th layer operates on its input feature matrix as follows:
\begin{equation}
    \mathbf{X}_{l+1} = \sigma(\hat{\mathbf{D}}^{-\frac{1}{2}}\hat{\mathbf{A}}\hat{\mathbf{D}}^{-\frac{1}{2}}\mathbf{X}_l\mathbf{W}_l),
\end{equation}
where $\hat{\mathbf{A}} = \mathbf{A} + \mathbf{I}$ is the adjacency matrix with self-loops, $\hat{\mathbf{D}}$ is the normalised node degree matrix and $\sigma$ is the activation function.

After $E$ layers, the embedding network simply outputs node features $\mathbf{X}_{L_E}$, which are subsequently processed by a pooling layer to coarsen the graph. The classification network first takes as input node features of the Mapper-pooled (Section \ref{mpool}) graph\footnote{Note that one or more \{embedding $\rightarrow$ pooling\} operations may be sequentially performed in the pipeline.}, $\mathbf{X}_{\text{MG}}$, and passes them through $L_C$ graph convolutional layers. Following this, the network computes a graph summary given by the feature-wise node average and applies a final linear layer which predicts the class:
\begin{equation}
    y = \text{softmax}\big(\frac{1}{N} \sum_{i=1}^N \mathbf{X}_{L_C}\mathbf{W_{\text{f}}} + \mathbf{b_{\text{f}}}\big),
\end{equation}
where $N$ is the number of nodes in the final graph.

\begin{figure*}[ht]
    \centering
    \begin{subfigure}[t]{0.23\textwidth}
         \centering
         \includegraphics[width=\textwidth]{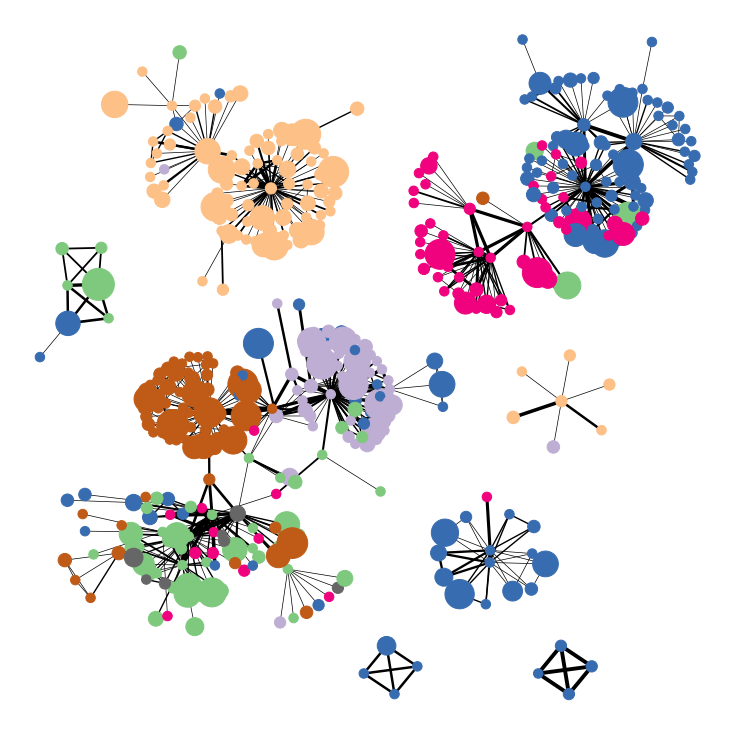}
         \caption{DGM on Cora}
    \end{subfigure}
    ~
    \begin{subfigure}[t]{0.23\textwidth}
         \centering
         \includegraphics[width=\textwidth]{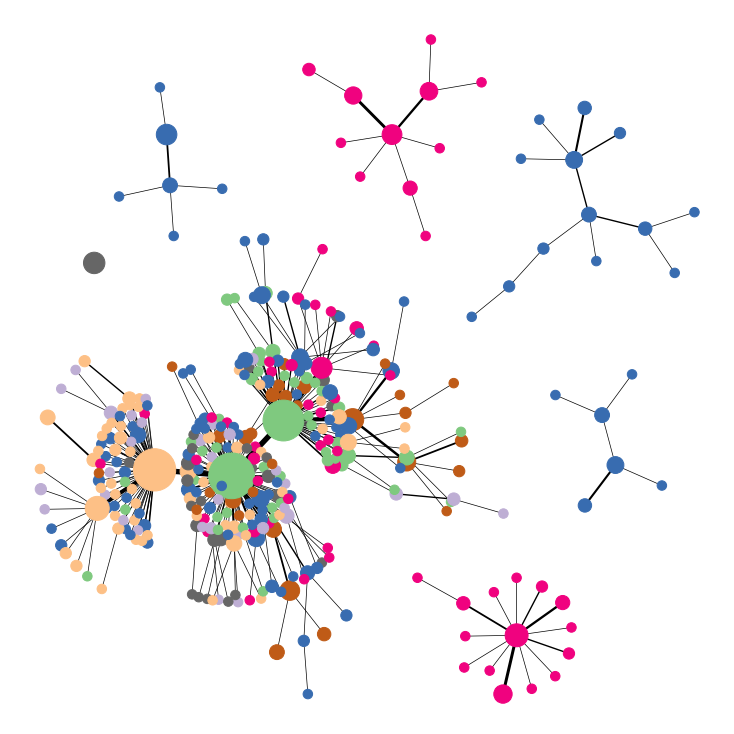}
         \caption{Density Mapper on Cora}
    \end{subfigure}
    ~
    \begin{subfigure}[t]{0.23\textwidth}
         \centering
         \includegraphics[width=\textwidth]{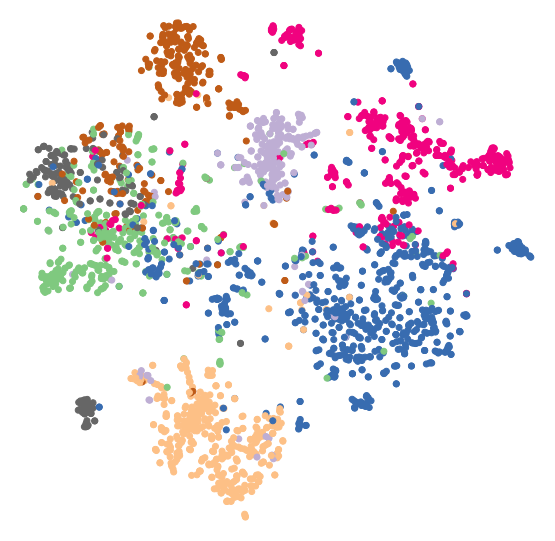}
         \caption{DGI $t$-SNE on Cora}
    \end{subfigure}
    ~
    \begin{subfigure}[t]{0.23\textwidth}
         \centering
         \includegraphics[width=\textwidth]{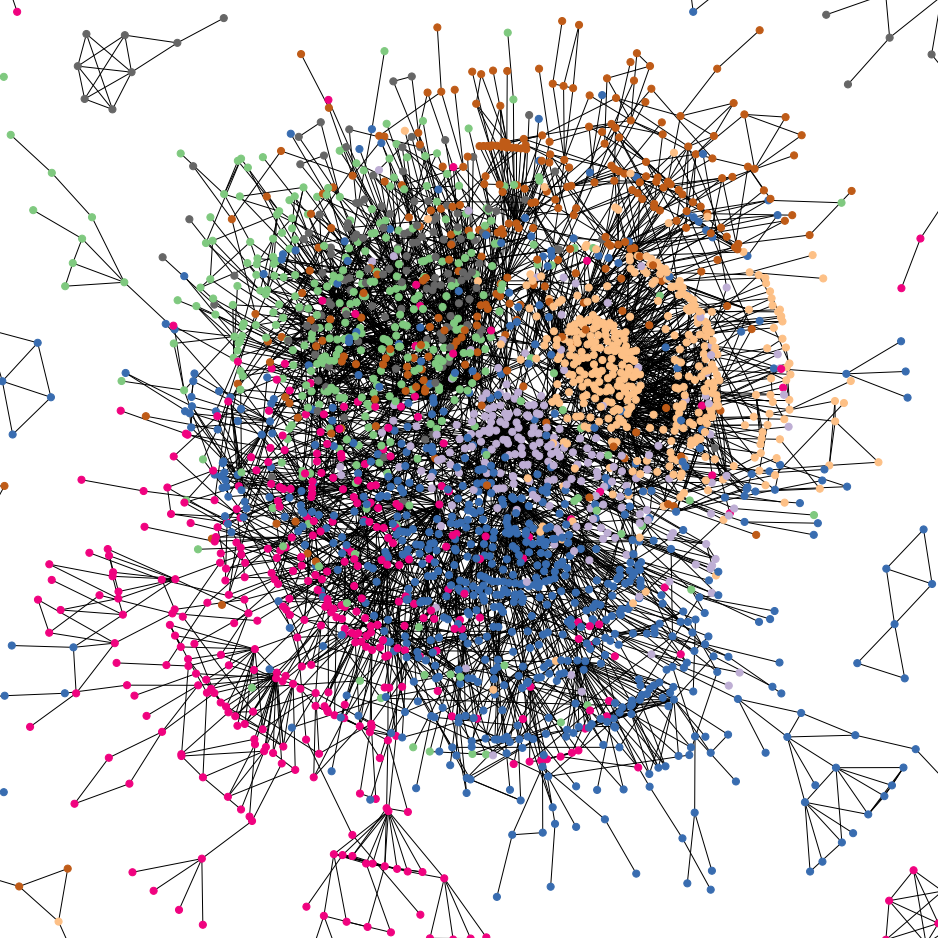}
         \caption{Graphviz layout on Cora}
    \end{subfigure}
    ~
    \begin{subfigure}[t]{0.23\textwidth}
         \centering
         \includegraphics[width=\textwidth]{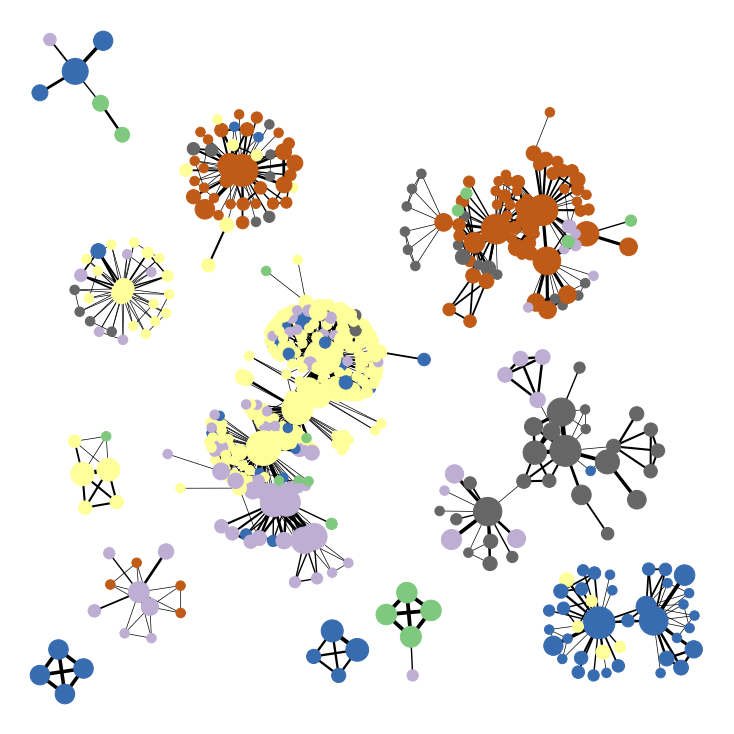}
         \caption{DGM on CiteSeer}
    \end{subfigure}
    ~
    \begin{subfigure}[t]{0.23\textwidth}
         \centering
         \includegraphics[width=\textwidth]{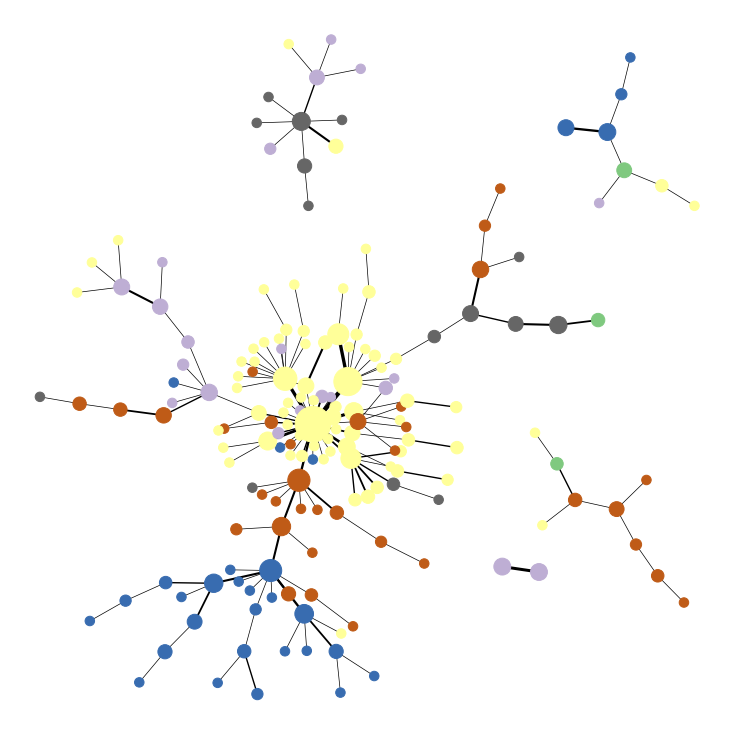}
         \caption{Density Mapper on CiteSeer}
    \end{subfigure}
    ~
    \begin{subfigure}[t]{0.23\textwidth}
         \centering
         \includegraphics[width=\textwidth]{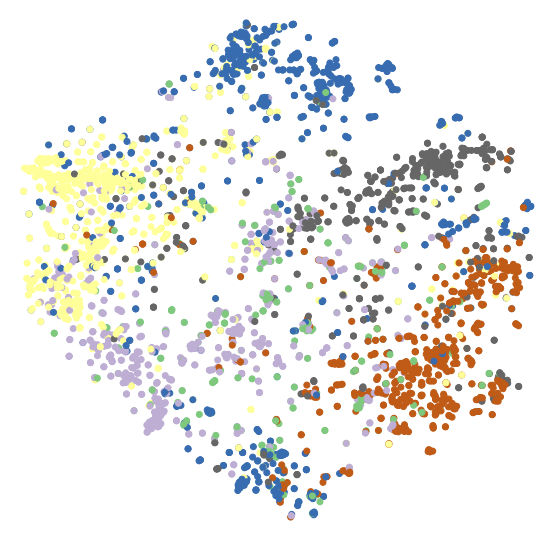}
         \caption{DGI $t$-SNE on CiteSeer}
    \end{subfigure}
      ~
    \begin{subfigure}[t]{0.23\textwidth}
         \centering
         \includegraphics[width=\textwidth]{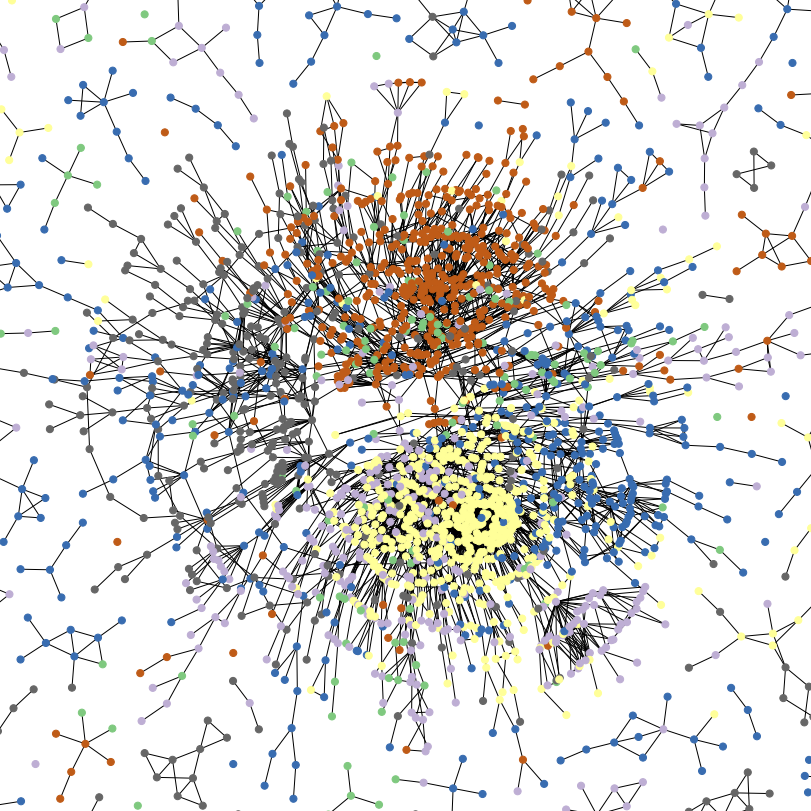}
         \caption{Graphviz layout on Citeseer}
    \end{subfigure}
    \caption{Qualitative comparison between DGM (a, e) and Mapper with an RBF graph density function (b, f). The DGI $t$-SNE plot (c, g) and the Graphviz visualisation of the full graph (d, h) are added for reference. The first and second row show plots for Cora and CiteSeer, respectively. The nodes are coloured with the dataset node labels. DGM with unsupervised lens implicitly makes all dataset classes appear in the visualisation clearly separated, which does not happen in the density visualisation. DGM also adds a new layer of information relative to the $t$-SNE plot by mapping the semantic information back to the original graph.}
    \label{fig:qualitative}
\end{figure*}

\subsubsection{Mapper-based PageRank (MPR) Pooling} \label{mpool}

We now describe the pooling mechanism used in our graph classification pipeline, which we adapt from the Mapper algorithm. The first step is to assign each node a real number in $[0,1]$, achieved by computing a lens function $f : V \rightarrow \mathbb{R}$ that is given by the normalised PageRank (PR)~\cite{page1999pagerank} of the nodes. The PageRank function assigns an importance value to each of the nodes based on their connectivity, according to the following recurrence relation:
\begin{equation} \label{eq:pr}
    f(\mathbf{X}_{L_E})_{i} \overset{\Delta}{=} \mathbf{PR}_i = \sum_{j \in N(i)} \frac{\mathbf{PR}_j}{| N(i) |},
\end{equation}
where $N(i)$ represents the set of neighbours of the $i$-th node in the graph. The resulting scores are values in $[0, 1]$ which reflect the probability of a random walk through the graph to end in a given node. Computing the $\mathbf{PR}$ vector is achieved using NetworkX~\citep{hagberg2008exploring} via power iteration, as $\mathbf{PR}$ is the principal eigenvector of the transition matrix $\mathcal{M}$ of the graph:
\begin{equation}
    \mathbf{PR} = \big( \alpha \mathcal{M} + (1-\alpha) \frac{1}{N} \mathbf{E} \big) \mathbf{PR},
\end{equation}
where $\mathbf{E}$ is a matrix with all elements equal to 1 and $\alpha \in [0, 1]$ is the probability of continuing the random walk at each step; a value closer to $0$ implies the nodes would receive a more uniform ranking and tend to be clustered in a single node. We choose the widely-adopted $\alpha = 0.85$ and refer the reader to~\citep{boldi2005pagerank} for more details.

We use the previously described overlapping intervals cover $\gU$ and determine the pull back cover induced by $(f, \gU)$. This effectively builds a soft cluster assignment matrix $\mathbf{S} \in \sR^{|\text{G}| \times |\text{MG}|}$ from nodes in the original graph to ones in the pooled graph:

\begin{equation}
    S_{ij} = \frac{\mathbb{I}_{i \in f^{-1}(U_j)}}{| \{U_k | i \in f^{-1}(U_k) \} |} 
\end{equation}

where $U_n$ is the $n$-th overlapping interval in the cover $\gU$ of $[0, 1]$. It can be observed that the resulting clusters contain nodes with similar PageRank scores, as determined by eq.~\ref{eq:pr}. Therefore, our pooling method intuitively merges the (usually few) highly connected nodes in the graph, and at the same time clusters the (typically many) dangling nodes that have a normalised PageRank score closer to zero.

Finally, the mapping $\mathbf{S}$ is used to compute features for the new nodes (i.e. the soft clusters formed by the pull back), $\mathbf{X}_\text{MG} = \mathbf{X}_{L_E} \mathbf{S}$, and the corresponding adjacency matrix, $\mathbf{A}_\text{MG} = \mathbf{S}^T\mathbf{A}\mathbf{S}$.

It is important that graph classification models are node permutation invariant since one graph can be represented by any tuple $(\mathbf{XP}, \mathbf{AP})$, where $\mathbf{P}$ is a node permutation. Bellow, we state a positive result in this regard for the MPR pooling procedure. 
\begin{proposition} \label{lemma:pr}
The PageRank pooling operator defined above is permutation-invariant.
\end{proposition}
\begin{proof}
First, we note that the PageRank function is permutation invariant and refer the reader to~\citet[Axiom 3.1]{altman2005pagerank} for the proof. It then follows trivially that the PageRank pooling operator is also permutation-invariant.
\end{proof}

\section{Experiments}

We now provide a qualitative evaluation of the DGM visualisations and benchmark MPR pooling on a set of graph classification tasks. 

\subsection{Tasks}

The DGM visualiser is evaluated on two popular citation networks: CiteSeer and Cora~\citep{sen2008collective}. We further showcase the applicability of DGM within a pooling framework, reporting its performance on a variety of settings: social (Reddit-Binary), citation networks (Collab) and chemical data (D\&D, Proteins)~\cite{KKMMN2016}.

\subsection{Qualitative Results}

In this section, we qualitatively compare DGM using a DGI lens against a Mapper instance that uses a fine-tuned graph density function $f(v)=\sum_{u \in V} e^{\frac{-d(u, v)}{\delta}}$ based on an RBF kernel (Figure~\ref{fig:qualitative}) with $d$ being the distance matrix of the graph. For reference, we also include a full-graph visualisation using a Graphviz layout and a t-SNE plot of the DGI embeddings that are used as the lens.

\subsection{Pooling Evaluation} \label{expsetup}

We adopt a 10-fold cross-validation approach to evaluating the graph classification performance of MPR and other competitive state-of-the-art methods. The random seed was set to $0$ for all experiments (with respect to dataset splitting, shuffling and parameter initialisation), in order to ensure a fair comparison across architectures. All models were trained on a Titan Xp GPU, using the Adam optimiser~\cite{kingma2014adam} with early stopping on the validation set, for a maximum of 30 epochs. We report the classification accuracy using 95\% confidence intervals calculated for a population size of 10 (the number of folds).

\subsubsection{Models} \label{models} We compare the performance of MPR to two other pooling approaches that we identify mathematical connections with---minCUT~\cite{bianchi2019mincut} and DiffPool~\cite{ying2018hierarchical}. Additionally, we include Graph U-Net~\cite{gao2019graph} in our evaluation, as it has been shown to yield competitive results, while performing pooling from the alternative perspective of a learnable node ranking; we denote this approach by \emph{Top-$k$} in the remainder of this section.

We optimise MPR with respect to its cover cardinality $n$, interval overlap percentage $g$ at each pooling layer, learning rate and hidden size. The Top-$k$ architecture is evaluated using the code provided in the official repository\footnote{\url{github.com/HongyangGao/Graph-U-Nets}}, where separate configurations are defined for each of the benchmarks. The minCUT architecture is represented by the sequence of operations described by~\citet{bianchi2019mincut}: \emph{MP(32)-pooling-MP(32)-pooling-MP(32)-GlobalAvgPool}, followed by a linear softmax classifier. The \emph{MP(32)} block represents a message-passing operation performed by a graph convolutional layer with 32 hidden units:
\begin{equation}
    \mathbf{X}^{(t+1)} = \text{ReLU}(\tilde{\mathbf{A}}\mathbf{X}^{(t)}\mathbf{W}_m + \mathbf{X}^{(t)}\mathbf{W}_s),
\end{equation}
where $\tilde{\mathbf{A}} = \mathbf{D}^{-\frac{1}{2}}\mathbf{A}\mathbf{D}^{-\frac{1}{2}}$ is the symmetrically normalised adjacency matrix and $\mathbf{W}_m, \mathbf{W}_s$ are learnable weight matrices representing the message passing and skip-connection operations within the layer. The DiffPool model follows the same sequence of steps. Full details of the model architectures and hyperparameters can be found in the supplementary material.

\subsubsection{Results} \label{results}

The graph classification performance obtain by these models is reported in Table~\ref{table:opt-results}. By suitably adjusting MPR hyperparameters, we achieve the best results for D\&D, Proteins and Collab and closely follow minCUT on Reddit-Binary. These results showcase the utility of Mapper for designing better pooling operators. 

\begin{table}[t]
\begin{center}
\resizebox{\columnwidth}{!}{%
\begin{tabular}{ ccccc } 
 \toprule
  & MPR & Top-$k$ & minCUT & DiffPool \\ 
 \midrule
 D\&D & $\mathbf{78.2 \pm 3.4}$ & $75.1 \pm 2.2$ & $77.6 \pm 3.1$ & $77.9 \pm 2.4$ \\ 
 Proteins & $\mathbf{75.2 \pm 2.2}$ & $74.8 \pm 3.04$ & $73.5 \pm 2.9$ & $74.2 \pm 0.3$ \\ 
 Collab & $\mathbf{81.5 \pm 1.0}$ & $75.0 \pm 1.1$ & $79.9 \pm 0.8$ & $81.3 \pm 0.1$ \\ 
 Reddit-B & $86.3 \pm 4.8$ & $74.9 \pm 7.4$ & $\mathbf{87.2 \pm 5.0}$ & $79.0 \pm 1.1$ \\ 
 \bottomrule
\end{tabular}
}
\caption{Results obtained by optimised architectures on classification benchmarks. Accuracy measures with 95\% confidence intervals are reported.}
\label{table:opt-results}
\end{center}
\end{table}

\section{Conclusion and Future Work}

We have introduced Deep Graph Mapper, a topologically-grounded method for producing informative hierarchical graph visualisations with the help of GNNs. We have shown these visualisations are not only helpful for understanding various graph properties, but can also aid in refining graph models. Additionally, we have proved that Mapper is a generalisation of soft cluster assignment methods, effectively providing a bridge between graph pooling and the TDA literature. Based on this connection, we have proposed a simple Mapper-based PageRank pooling operator, competitive with several state-of-the-art methods on graph classification benchmarks. Future work will focus on back-propagating through the pull back computation to automatically learn a lens and cover. Lastly, we plan to extend our methods to spatio-temporally evolving graphs.

\section*{Acknowledgement}

We would like to thank Petar Veli\v{c}kovi\'c, Ben Day, Felix Opolka, Simeon Spasov, Alessandro Di Stefano, Duo Wang, Jacob Deasy, Ramon Vi\~nas, Alex Dumitru and Teodora Reu for their constructive comments. We are also grateful to Teo Stoleru for helping with the diagrams.

\bibliography{example_paper}
\bibliographystyle{icml2020}

\cleardoublepage

\begin{figure*}[!ht]
    \centering
    \includegraphics[width=0.9\textwidth]{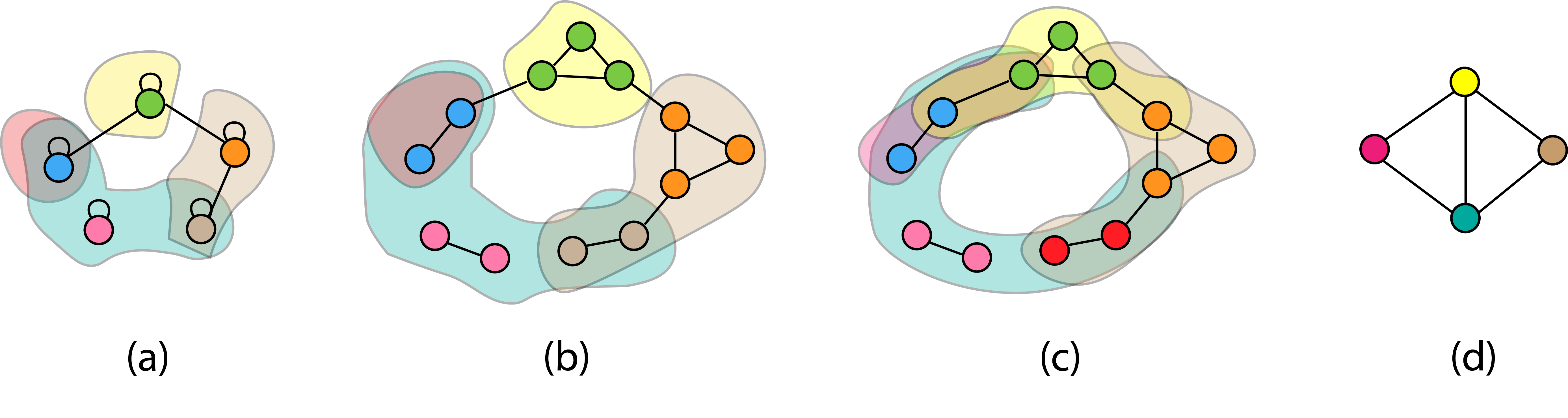}
    \caption{An illustration of the proof for Proposition 4.2. Connecting the clusters connected by an edge in the original graph (a) is equivalent to constructing an expanded graph with an expanded cluster assignment (b), doing a 1-hop expansion of the expanded cluster assignment (c) and finally taking the 1-skeleton of the nerve (d) of the cover obtained in (c).}
    \label{fig:mapper_proof}
\end{figure*}

\section*{Appendix A: Proof of Proposition 4.2}
\label{app:mapper_proof}

Throughout the proof we consider a graph $G(V, E)$ with self loops for each of its $N$ nodes. The self-loop assumption is not necessary, but it elegantly handles a number of degenerate cases involving nodes isolated from the other nodes in its cluster. We refer to the edges of a node which are not a self-loop as \textbf{external edges}.

Let $s: V \to \bigtriangleup_{K-1}$ be a soft cluster assignment function that maps the vertices to the $(K-1)$-dimensional unit simplex. We denote by $s_k(v)$ the probability that vertex $v$ belongs to cluster $k \leq K$ and $\sum_k^K s_k(v) = 1$. This function can be completely specified by a cluster assignment matrix $\mathbf{S} \in \sR^{N \times K}$ with $S_{ik} = s_k(i)$. This is the soft cluster assignment matrix computed by minCut and Diff pool via a GCN. 

\begin{definition}
Let $G(V, E)$ be a graph with self-loops for each node. The expanded graph $G'(V', E')$ of $G$ is the graph constructed as follows. For each node $v \in V$ with at least one external edge, there is a clique of nodes $\{v_1, \ldots, v_{deg(v)}\} \subseteq V'$. For each external edge $(v, u) \in E$, a pair of nodes from their corresponding cliques without any edges outside their clique are connected. Additionally, isolated nodes become in the new graph two nodes connected by an edge. 
\end{definition}

Essentially, the connections between the nodes in the original graph are replaced by the connections between the newly formed cliques such that every new node is connected to at most one node outside its clique. An example of an expanded graph is shown in Figure \ref{fig:mapper_proof} (b). The nodes in the newly formed cliques are coloured similarly to node from the original graph they originate from. 

\begin{definition}
Let $G(V, E)$ be a graph with self loops for each node, $s$ a soft cluster assignment function for it, and $G'(V', E')$ the expanded graph of $G$. Then the expanded soft cluster assignment $s^*$ is a cluster assignment function for $G'$, where $s^*_k(v_i) = s_k(v)$ for all the nodes $v_i \in V'$ in the corresponding clique of $v \in V$.
\end{definition}

In plain terms, all the nodes in the clique inherit through $s^*$ the cluster assignments of the corresponding node from the original graph. This is also illustrated by the coloured contours of the expanded graph in Figure \ref{fig:mapper_proof} (b). 

\begin{definition}
Let $\mS$ be a soft cluster assignment matrix for a graph $G(V, E)$ with adjacency matrix $\mA$. The \textbf{1-hop expansion} of assignment $\mathbf{S}$ with respect to graph $G$ is a new cluster assignment function $s': V \to \bigtriangleup_{K-1}$ induced by the row-normalised version of the cluster assignment matrix $\mathbf{S}'= \mathbf{A}\mathbf{S}$. 
\end{definition}

As we shall now prove, the 1-hop expansion simply extends each soft cluster from $\mathbf{S}$ by adding its 1-hop neighbourhood as in Figure \ref{fig:mapper_proof} (c).

\begin{lemma}
\label{lemma:hop}
An element of the soft cluster assignment matrix $S'_{i,k} \neq 0$ if and only if node $i$ is connected to a node $j$ (possibly $i=j$), which is part of the soft cluster $k$ of the assignment induced by $\mathbf{S}$.
\end{lemma}

\begin{proof}
By definition $S'_{ik} = \sum_j^N A_{ij} S_{jk} = 0$ if and only if $A_{ij} S_{jk} = 0$, for all $j$. This can happen if and only if $A_{ij} = 0$ (nodes $i$ and $j$ are not connected by an edge) or $S_{jk} = 0$ (node $j$ does not belong to soft cluster $k$ defined by $\mathbf{S}$) for all $j$. Therefore, $S'_{ik} \neq 0$ if and only if there exists a node $j$ such that $i$ is connected to $j$ and $j$ belongs to soft cluster $k$ defined by $\mathbf{S}$.
\end{proof}

\begin{corollary}
\label{corrolary:self}
Nodes that are part of a cluster $k$ defined by $\mathbf{S}$, are also part of $k$ under the assignment $\mathbf{S}'$.
\end{corollary}

\begin{proof}
This immediately follows from Lemma \ref{lemma:hop} and the fact that each node has a self-loop. 
\end{proof}

\begin{lemma}
\label{lemma:soft_clust}
The adjacency matrix $\mathbf{A'} = \mathbf{S^T}\mathbf{A}\mathbf{S}$ defines a new graph where the clusters induced by $S$ are connected if and only if there is a common edge between them.
\end{lemma}

\begin{proof}
Let $\mathbf{L} = \mathbf{A}\mathbf{S}$. Then, $A'_{ij} = \sum_k^N S^{T}_{ik} L_{kj} = 0$ if and only if $S^{T}_{ik} = 0$ (node $k$ does not belong to cluster $i$) or $L_{kj} = 0$ (node $k$ is not connected to any node belonging to cluster $j$ by Lemma \ref{lemma:hop}), for all $k$. Therefore, $A'_{ij} \neq 0$ if and only if there exists a node $k$ such that $k$ belongs to cluster i and $k$ is connected to a node from cluster $j$. 
\end{proof}

This result shows that soft cluster assignment methods connect clusters that have at least one edge between them. We will use this result to show that a Mapper construction obtains an identical graph. 

Let $s'$ be the 1-hop expansion of the expanded soft cluster assignment of graph $G'$. Let the soft clusters induced by $s'$ be $\gA = \{A_1, A_2, \ldots, A_K\}$. Additionally, let $\gB = \{B_1, B_2, \dots, B_K\}$ be an open cover of $\bigtriangleup_{K-1}$ with $B_k = \{\mathbf{x} \in \bigtriangleup_{K-1} | \mathbf{x}_k > 0\}$. Then the pull back cover induced by $(s', \gB)$ is $\gA$ since $s'^{(-1)}(B_k) = A_k$ (i.e. all nodes with a non-zero probability of belonging to cluster k). 

\begin{lemma}
\label{lemma:nerve}
Two clusters $A_x$ and $A_y$ have a non-empty intersection in the expanded graph if and only if their corresponding clusters $C_x$, $C_y$ in the original graph have a common edge between them. 
\end{lemma}

\begin{proof}
If direction: By Corollary \ref{corrolary:self}, the case of a self edge becomes trivial. Now, let $v \in C_x$ and $u \in C_y$ be two nodes connected by an external edge in the original graph. Then in the expanded graph $G'(V', E')$, there will be clique nodes $v_i, u_i \in V'$ such that $(v_i, u_i) \in E'$. By taking the 1-hop expansion of the extended cluster assignment, both $v_i$ and $v_j$ will belong to $A_x, A_y$ by Lemma \ref{lemma:hop}, since they are in each other's 1-hop neighbourhood. Since we have chosen the clusters and the nodes arbitrarily, this proves this direction. 

Only if direction: Let $C^*_x$ and $C^*_y$ be the (expanded) clusters in the expanded graph corresponding to the clusters $C_x$ and $C_y$ in the original graph. Let node $v_i$ be part of the non-empty intersection between soft clusters $A_x$ and $A_y$ defined by $\mS'$ in the expanded graph $G'$. By Lemma \ref{lemma:hop}, $v_i$ belongs to $A_x$ if and only if there exists $v_j \in V'$ such that $(v_i, v_j) \in E'$ and $v_j \in C^*_x$. Similarly, there must exist a node $v_k \in V'$ such that $(v_i, v_k) \in E'$ and $v_k \in C^*_y$. By the construction of $G'$, either both $v_j, v_k$ are part of the clique $v_i$ is part of, or one of them is in the clique, and the other is outside the clique. 

Suppose without loss of generality that $v_j$ is in the clique and $v_k$ is outside the clique. Then, $v_i \in C^*_x$ since they share the same cluster assignment. By the construction of $G'$ the edge between the corresponding nodes of $v_i$ and $v_k$ in the original graph $G$ is an edge between $C_x$ and $C_y$. Similarly, if both $v_j$ and $v_k$ are part of the same clique with $v_i$, then they all originate from the same node $v \in C_x, C_y$ in the original graph. The self-edge of $v$ is an edge between $C_x$ to $C_y$.  
\end{proof}

\begin{proposition}
Let $G(V, E)$ be a graph with self loops and adjacency matrix $\mathbf{A}$. The graphs defined by $sk_1(\gN(\gA))$ and $\mathbf{S}^T\mathbf{A}\mathbf{S}$ are isomorphic.
\end{proposition}

\begin{proof}
Based on Lemma \ref{lemma:nerve}, $sk_1(\gN(\gA))$ connects two soft clusters in $G$ defined by $\mS$ if and only if there is a common edge between them. By Lemma \ref{lemma:soft_clust}, soft cluster assignment methods connect the soft clusters identically through the adjacency matrix $\mathbf{A'} = \mathbf{S^T}\mathbf{A}\mathbf{S}$. Therefore, the unweighted graphs determined by $sk_1(\gN(\gA))$ and $\mathbf{A'}$ are isomorphic. 
\end{proof}

Note that our statement refers strictly to unweighted edges. However, an arbitrary weight function $w: V \to \sR$ can easily be attached to the graph obtained though $sk_1(\gN(\gA))$.

\section*{Appendix B: Structural Deep Graph Mapper}

The edges between the nodes in a DGM visualisation denote semantic similarities discovered by the lens. However, even though semantically-related nodes are often neighbours in many real-world graphs, this is not always the case. For example, GNNs have been shown to ignore this assumption, often relying entirely on the graph features~\citep{Luzhnica2019OnGC}. 

\begin{figure}[!ht]
    \centering
    \includegraphics[width=0.8\columnwidth]{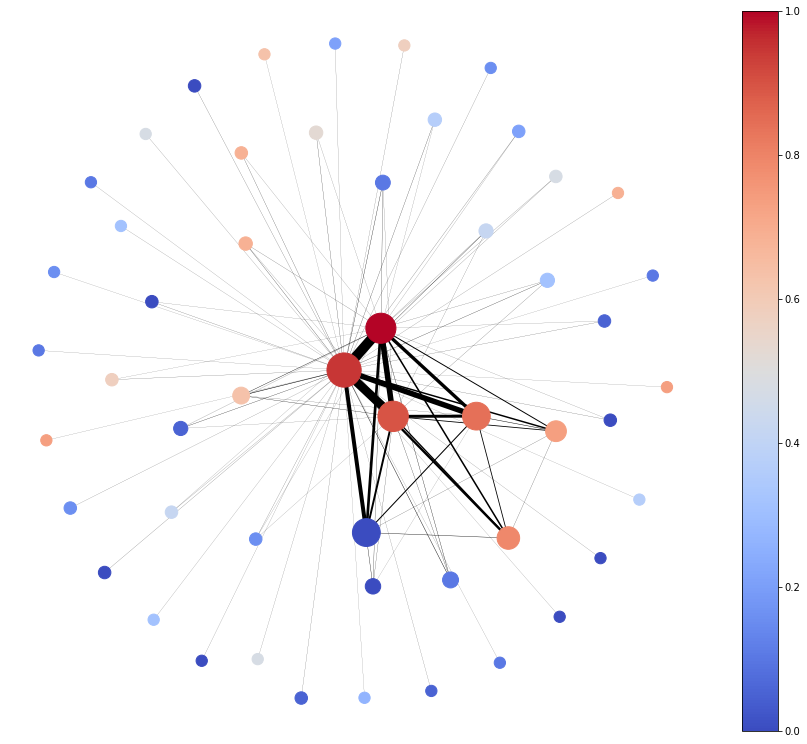}
    \caption{SDGM visualisation of the Spammer dataset. The thickness of the edges is now proportional to the number of edges between clusters. We used a filtration value of $\eps = 0.01$ and $g = 0$ for the overlap. This visualisation also illustrates that the spammers are densely connected to the other nodes in the graph, while non-spammers form smaller connected groups. However, unlike the DGM visualisation, this graph also shows the (structural) edges between the spammers and the non-spammers.}
    \label{fig:intro_sdgm}
\end{figure}

\begin{figure*}[!ht]
    \centering
    \begin{subfigure}[t]{0.32\textwidth}
         \centering
         \includegraphics[width=\textwidth]{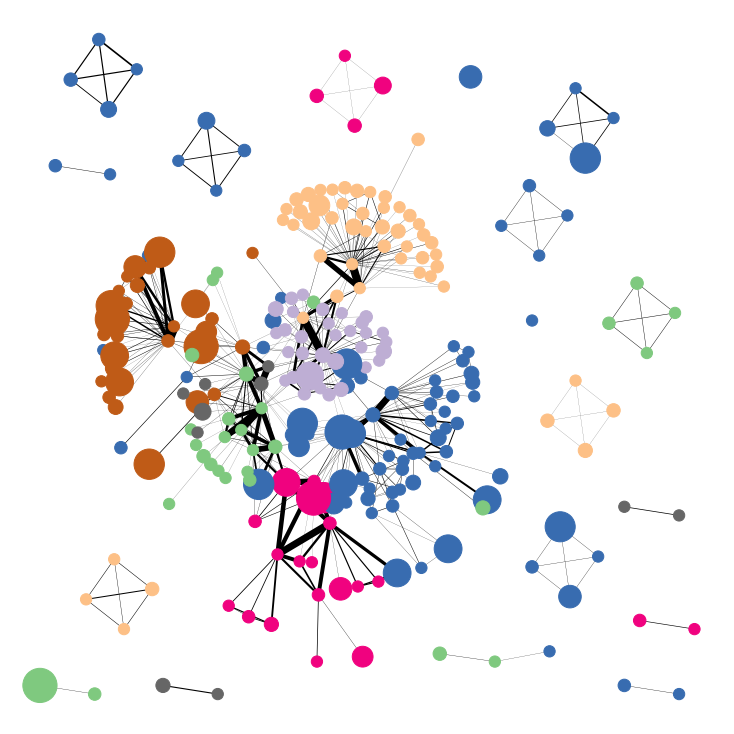}
         \caption{$\eps = 0.01$}
    \end{subfigure}
    ~
    \begin{subfigure}[t]{0.32\textwidth}
         \centering
         \includegraphics[width=\textwidth]{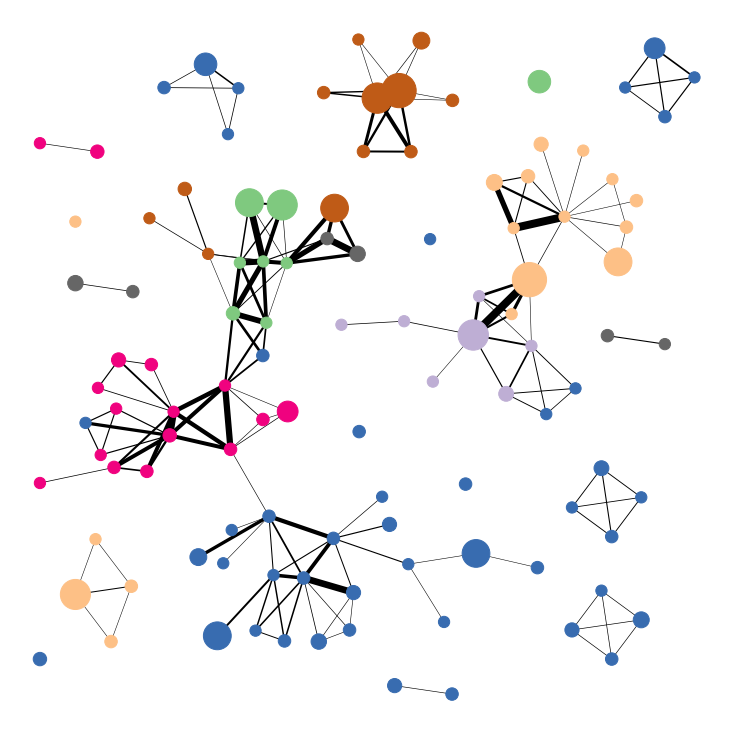}
         \caption{$\eps = 0.05$}
    \end{subfigure}
    ~
    \begin{subfigure}[t]{0.33\textwidth}
         \centering
         \includegraphics[width=\textwidth]{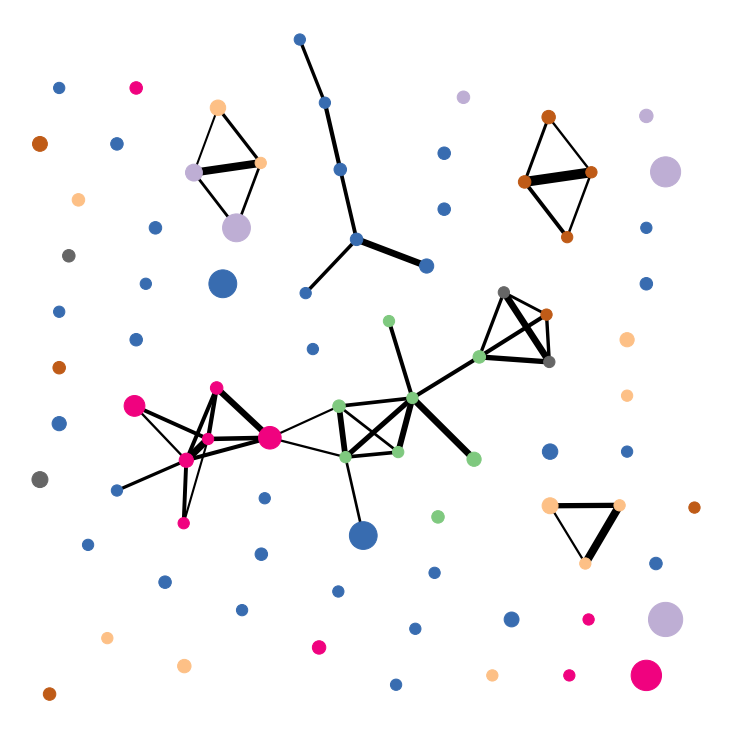}
         \caption{$\eps = 0.10$}
    \end{subfigure}
    ~
    \caption{SDGM visualisation on Cora using DGI lens and ground-truth labels, with varying values of $\eps$. Lower values of $\eps$ increase the connectivity of the graph.}
    \label{fig:sdgm_multi}
\end{figure*}

Therefore, it is sometimes desirable that the connectivity of the graph is explicitly accounted for in the visualisations, being involved in more than simply computing the refined pull back cover. Motivated by the proof from Appendix A, we also propose~\textbf{Structural DGM (SDGM)}, a version of DGM that connects the elements of the refined pull back cover based on the number of edges between the component nodes from the original graph.

SDGM uses the refined pull back cover induced by the 1-hop expansion in the expanded graph (see Appendix A) to compute the nerve. However, a downside of this approach is that the output graph may often be too dense to visualise. Therefore, we use a weighting function $w: E \to [0, 1]$ to weight the edges of the resulting graph and a filtration value $\eps \in [0, 1]$. We then filter out all the edges $e$ with $w(e) < \eps$, where $w(e)$ is determined by the normalised weighted adjacency matrix $\mathbf{S^T}\mathbf{A}\mathbf{S}$ denoting the (soft) number of edges between two clusters. Figure~\ref{fig:intro_sdgm} includes an SDGM visualisation for the spammer graph.

The overlap parameter $g$ effectively sets a trade-off between the number of structural and semantic connections in the SDGM graph. For $g = 0$, only structural connections exist. At the same time, the filtration constant $\eps$ is an additional parameter that adjusts the resolution of the visualisations. Higher values of $\eps$ result in sparser graphs, while lower values increase the connectivity between the nodes. We illustrate the effects of varying $\eps$ in Figure~\ref{fig:sdgm_multi}.

\section*{Appendix C: Model Architecture and Hyperparameters}

The optimised models described in the Experiments section have the following configurations:
\begin{itemize}
    \item DGM---learning rate $5e^{-4}$, hidden sizes $\{128, 128\}$ and:
    \begin{itemize}
        \item D\&D and Collab: cover sizes $\{20, 5\}$, interval overlap $10\%$, batch size $32$;
        \item Proteins: cover sizes $\{8, 2\}$, interval overlap $25\%$, batch size $128$;
        \item Reddit-Binary: cover sizes $\{20, 5\}$, interval overlap $25\%$, batch size $32$;
    \end{itemize}
    \item Top-$k$---specific dataset configurations, as provided in the official GitHub repository (\texttt{run\_GUNet.sh});
    \item minCUT---learning rate $1e^{-3}$, same architecture as reported by the authors in the original work~\citep{bianchi2019mincut};
    \item DiffPool---learning rate $1e^{-3}$, hidden size $32$, two pooling steps, pooling ratio $r=0.1$, global average mean readout layer, with the exception of Collab and Reddit-Binary, where the hidden size was $128$.
\end{itemize}

We additionally performed a hyperparameter search for DiffPool on hidden sizes ${32, 64, 128}$ and for DGM, over the following sets of possible values:
\begin{itemize}
    \item all datasets: cover sizes $\{[40, 10], [20, 5]\}$, interval overlap $\{10\%, 25\%\}$;
    \item D\&D: learning rate $\{5e^{-4}, 1e^{-3}\}$;
    \item Proteins: learning rate $\{2e^{-4}, 5e^{-4}, 1e^{-3}\}$, cover sizes $\{[24, 6], [16, 4], [12, 3], [8, 2]\}$, hidden sizes $\{64, 128\}$.
\end{itemize}

\end{document}